\documentclass[letterpaper]{article} 
\usepackage{aaai23}  
\usepackage{times}  
\usepackage{helvet}  
\usepackage{courier}  
\usepackage[hyphens]{url}  
\usepackage{graphicx} 
\usepackage{natbib}  
\usepackage{caption} 
\usepackage{algorithm}
\usepackage{algorithmic}
\usepackage{kotex}
\usepackage{url}            
\usepackage{booktabs}       
\usepackage{amsfonts}       
\usepackage{nicefrac}       
\usepackage{microtype}      
\usepackage{xcolor}         
\usepackage{multirow}
\usepackage{mathtools}
\usepackage{subfigure}
\usepackage{diagbox}
\usepackage{amsmath}
\usepackage{amsthm}
\usepackage{bm}
\usepackage{physics}
\newtheorem{theorem}{Theorem}
\newtheorem{corollary}{Corollary}
\usepackage{newfloat}
\usepackage{listings}
\DeclareCaptionStyle{ruled}{labelfont=normalfont,labelsep=colon,strut=off} 
\floatstyle{ruled}
\newfloat{listing}{tb}{lst}{}
\floatname{listing}{Listing}
\title{Regular Time-series Generation using SGM\thanks{This paper is under review.}}
\author{
    Haksoo Lim\textsuperscript{\rm 1}, Minjung Kim\textsuperscript{\rm 2}, Sewon Park\textsuperscript{\rm 2}, Noseong Park\textsuperscript{\rm 1}
}
\affiliations{
    \textsuperscript{\rm 1}Yonsei University,
    \textsuperscript{\rm 2}Samsung SDS

    limhaksoo96@yonsei.ac.kr, mj100.kim@samsung.com,\\ sw0413.park@samsung.com, noseong@yonsei.ac.kr
}

\usepackage{bibentry}

\begin{document}

\maketitle

\begin{abstract}
Score-based generative models (SGMs) are generative models that are in the spotlight these days. Time-series frequently occurs in our daily life, e.g., stock data, climate data, and so on. Especially, time-series forecasting and classification are popular research topics in the field of machine learning. SGMs are also known for outperforming other generative models. As a result, we apply SGMs to synthesize time-series data by learning conditional score functions. We propose a conditional score network for the time-series generation domain. Furthermore, we also derive the loss function between the score matching and the denoising score matching in the time-series generation domain. Finally, we achieve state-of-the-art results on real-world datasets in terms of sampling diversity and quality.
\end{abstract}

\noindent 

\section{Introduction}

SGMs now show remarkable results throughout various fields, including image generation, voice synthesis, etc.~\cite{ahmed2010timeforecast,fu2011time, fawaz2019dltime}. Let $\{(\mathbf{x}_i, t_i)\}_{i=1}^N$ be a time-series sample which consists of $N$ observations. In many cases, however, time-series samples are incomplete and/or the number of samples is insufficient, in which case training machine learning models cannot be fulfilled in a robust way. To overcome the limitation, time-series synthesis has been studied actively recently. These synthesis models have designed in various ways, including variational autoencoder (VAE) and generative adversarial network (GAN)~\cite{desai2021timevae, yoon2019timegan}.

In the image generation field, score-based generative models (SGMs) (or diffusion models) have recently attracted much attention for their superior sampling quality and diversity over other generative paradigms ~\cite{song2021SDE,xiao2022tackling}. Being inspired by their results, many researchers attempted to apply SGMs to other generation tasks, such as medical image processing and voice synthesis~\cite{DBLP:journals/corr/abs-2103-16091,song2022solving}, etc.

\begin{table}[t]
\centering
\resizebox{0.98\columnwidth}{!}{
\begin{tabular}{c|c|c|c}
    \hline
    Method & Domain & Type  & Target score function \\
    \hline
    TimeGrad & Forecasting & Diffusion & ${\nabla}\text{log}p(\textbf{x}_t^s|\textbf{h}_{t-1})$\\
    ScoreGrad & Forecasting & Energy & ${\nabla}\text{log}p(\textbf{x}_t^s|\textbf{x}_t)$\\
    CSDI & Imputation & Diffusion & ${\nabla}\text{log}p(\textbf{x}_{ta}^s|\textbf{x}_{co})$\\
    \hline
    TSGM (Ours) & Generation & SGM & ${\nabla}\text{log}p(\textbf{h}_{t}^s|\textbf{h}_{t})$\\
    \hline
\end{tabular}}
\caption{Qualitative comparison among score-based methods for time-series}
\label{table4}
\end{table}

Although there exist several efforts to generate time-series, according to our survey, there is no research using SGMs --- there exist time-series SGMs only for forecasting and imputation~\cite{rasul2021timegrad, yan2021scoregrad, tashiro2021csdi} (cf. Table~\ref{table4}). Therefore, we extend SGMs into the field of time-series synthesis. Our time-series synthesis is technically different from time-series forecasting, which forecast future observations given past observations, and time-series imputation, which given a time-series sample with missing observations, fills out those missing observations. We discuss about the differences in the related work section in detail. Unlike the image generation which generates each image independently, in addition, time-series generation must generate each observation considering its past generated observations, i.e., conditional sampling. To this end, in this paper, we propose the method of time-series generation using conditional score-based generative model (TSGM).

We design the first conditional score network on time-series generation, which learns the gradient of the conditional log-likelihood with respect to time. We also design a denoising score matching on time-series generation --- existing SGM-based time-series forecasting and imputation methods also have their own denoising score matching definitions, but our denoising score matching differs from theirs due to the task difference (cf. Table~\ref{table4}). Our TSGM can be further categorized into two types depending on the used stochastic differential equation type: VP, and subVP. 

We conduct in-depth experiments with 5 real-world datasets and 2 key evaluation metrics --- therefore, there are, in total, 10 different evaluation cases. Our specific choice of 8 baselines includes almost all existing types of time-series generative paradigms, ranging from VAEs to GANs. Our proposed method shows the best generation quality in all but two cases. We also visualize real and generated time-series samples onto a latent space using t-SNE~\cite{maaten2008tsne} and those visualization results intuitively show that our method's generation diversity is also the best among those baselines. Our contributions can be summarized as follows:
\begin{enumerate}
    \item We, for the first time, propose a SGM-based time-series synthesis method (although there exist SGM-based time-series forecasting and imputation methods).
    \item We derive our own denoising score matching mechanism which is different other existing ones, considering the fully recurrent nature of our time-series generation.
    \item We conduct comprehensive experiments with 5 real-world datasets and 8 baselines. Overall, our proposed method shows the best generation quality and diversity.
\end{enumerate}

\section{Related Work and Preliminaries}

We review the literature for SGMs and time-series generation, followed by preliminary knowledge.

\subsection{Score-based Generative Models}

SGMs are one of the most popular image generation methods. It has several advantages over other generative models for its generation quality, computing exact log-likelihood, and controllable generation without extra training. For its remarkable results, lots of researchers try to apply them to other fields, e.g., voice synthesis~\cite{DBLP:journals/corr/abs-2103-16091}, medical image process~\cite{song2022solving}, etc. SGMs consists of the following two procedures: i) they first add Gaussian noises into a sample, ii) and then remove the added noises to recover new sample. Those two processes are called as forward and reverse process, respectively.

\subsubsection{Forward Process}

At first, SGMs add noises with the following stochastic differential equation (SDE):
\begin{align}
    d\textbf{x}^s=\textbf{f}(s,\textbf{x}^s)ds+g(s)d\textbf{w}, \qquad s \in [0,1],
\end{align}
where ${\textbf{w} \in \mathbb{R}^n}$ is $n$ dimensional Brownian motion, $\textbf{f}(s,\cdot) : \mathbb{R}^n \rightarrow \mathbb{R}^n$ is vector-valued drift term, and $g : [0,1] \rightarrow \mathbb{R}$ is scalar-valued diffusion functions. Here after, we define $\textbf{x}^s$ as a noisy sample diffused at time $s \in [0,1]$ from an original sample $\textbf{x} \in \mathbb{R}^n$. Therefore, $\textbf{x}^s$ can be understood as a stochastic process following the SDE. There are several options for $\textbf{f}$ and $g$: variance exploding(VE), variance preserving(VP), and subVP. ~\cite{song2021SDE} proved that VE and VP are continuous generalizations of the two discrete diffusion methods,~\cite{sohl2015ddpm, ho2020ddpm}  and ~\cite{song2019smld}. In addition, the author further suggested the subVP method, which is a modified version of VP.

SGMs run the forward SDE with sufficiently large $N$ steps to make it sure that the diffused sample converges to a Gaussian distribution at the final step. The score network $M_{\theta}(s,\textbf{x}^s)$ learns the gradient of the log-likelihood $\nabla_{\textbf{x}^s} \text{log}p(\textbf{x}^s)$, which will be used in the reverse process.

\subsubsection{Reverse Process}

For each forward SDE from $s=0$ to $1$,~\cite{anderson1982reverse} proved that there exists the following corresponding reverse SDE:
\begin{align*}
    d\textbf{x}^s=[\textbf{f}(s,\textbf{x}^s)-g^2(s)\nabla_{\textbf{x}^s}{\text{log}p(\textbf{x}^s)}]ds+g(s)d\bar{\textbf{w}}.
\end{align*}

The formula suggests that if knowing the score function, $\nabla_{\textbf{x}^s}{\text{log}p(\textbf{x}^s)}$, we can recover real samples from the prior distribution $p_1(\textbf{x}) \sim \mathcal{N}(\mu, \sigma^2)$, where $\mu, \sigma$ vary depending on the forward SDE type.

\subsubsection{Training and Sampling}

In order for the model $M$ to learn the score function, the model has to optimize the following loss function:
\begin{align*}
&L(\theta) = \mathbb{E}_{s}\{\lambda(s)\mathbb{E}_{\textbf{x}^s}[\left\|M_{\theta}(s,\textbf{x}^s)-{\nabla}_{\textbf{x}^s}\text{log}p(\textbf{x}^s)\right\|_2^2]\},
\end{align*}where $s$ is uniformly sampled over $[0,1]$ with an appropriate weight function $\lambda(s):[0,1]\rightarrow \mathbb{R}$.
However, using the above formula is problematic since we do not know the exact gradient of the log-likelihood. Thanks to~\cite{vincent2011matching}, the loss can be substituted with the following denoising score matching loss:
\begin{align*}
L^*(\theta)=\mathbb{E}_{s}\{\lambda(s)\mathbb{E}_{\textbf{x}^0}\mathbb{E}_{\textbf{x}^s|\textbf{x}^0}[\left\|M_{\theta}(s,\textbf{x}^s)-{\nabla}_{\textbf{x}^s}\text{log}p(\textbf{x}^s|\textbf{x}^0)\right\|_2^2]\}.
\end{align*}
Since SGMs use an affine drift term, the transition kernel $\text{p}(\textbf{x}^s|\textbf{x}^0)$ follows a certain Gaussian distribution~\cite{sarkka2019applied} and therefore, ${\nabla}_{\textbf{x}^s}\text{log}p(\textbf{x}^s|\textbf{x}^0)$ can be analytically calculated.

\subsection{Time-series Generation and SGMs}
\subsubsection{Time-series Generation}

In order to synthesize time-series $\textbf{x}_{1:T}$, unlike other generation tasks, we must generate each observation $\textbf{x}_t$ at time $t \in [2:T]$ considering its previous history $\textbf{x}_{1:t-1}$. One can train neural networks to learn the conditional likelihood $\text{p}(\textbf{x}_t|\textbf{x}_{1:t-1})$ and generate each $\textbf{x}_t$ recursively using it. There are several time-series generation papers, and we introduce their ideas.

TimeVAE~\cite{desai2021timevae} is a variational autoencoder to synthesize time-series data. This model can provide interpretable results by reflecting temporal structures such as trend and seasonality in the generation process. TimeGAN~\cite{yoon2019timegan} uses a GAN architecture to generate time-series. First, it trains an encoder and decoder, which transform a time-series sample $\textbf{x}_{1:T}$ into latent vectors $\textbf{h}_{1:T}$ and recover them by using a recurrent neural network (RNN). Next, it trains a generator and discriminator pair on latent space, by minimizing discrepancy between ground-truth  ${p}(\textbf{x}_t|\textbf{x}_{1:t-1})$ and synthesized $\hat{p}(\textbf{x}_t|\textbf{x}_{1:t-1})$. Since it uses an RNN-based encoder, it can efficiently learn the conditional likelihood ${p}(\textbf{x}_t|\textbf{x}_{1:t-1})$ by treating it as ${p}(\textbf{h}_t|\textbf{h}_{t-1})$, since $\textbf{h}_t\sim\textbf{x}_{1:t}$ under the regime of RNNs. Therefore, it can generate each observation $\textbf{x}_t$ considering its previous history $\textbf{x}_{1:t-1}$. However, GAN-based generative models are vulnerable to the issue of mode collapse~\cite{xiao2022tackling} and unstable behavior problems during training~\cite{chu2020unstable}. There also exist GAN-based methods to generate other types of sequential data, e.g., video, sound, etc~\cite{esteban2017rcgan,mogren2016crnngan,xu2020cotgan,donahue2019wavegan}. In our experiments, we also use them as our baselines for thoroughly evaluating our method.

\subsubsection{SGMs on Time-series}

Although there exist a little research work using their own conditional score networks, TimeGrad~\cite{rasul2021timegrad} and ScoreGrad~\cite{yan2021scoregrad} are for time-series forecasting, and CSDI~\cite{tashiro2021csdi} is for time-series imputation.

In TimeGrad~\cite{rasul2021timegrad}, the authors used a diffusion model, which is a discrete version of SGMs, to forecasting future observations given past observations by minimizing the following objective function:
\begin{align*}
    \sum_{t=t_0}^T -\text{log}p_{\theta}(\textbf{x}_t|\textbf{x}_{1:t-1},\textbf{c}_{1:T}),
\end{align*}
where $\textbf{c}_{1:T}$ means the covariates of $\textbf{x}_{1:T}$. The above formula assume that we already know $\textbf{x}_{1:t_0-1}$, and by using an RNN encoder, $(\textbf{x}_{1:t}$,$\textbf{c}_{1:T})$ can be encoded into $\textbf{h}_t$. After training, the model forecasts future observations recursively. After taking  $\textbf{x}_{1:t}$, it encodes with $\textbf{c}_{1:T}$ into $\textbf{h}_t$, and forecast the next step observation $\textbf{x}_{t+1}$ from the previous condition $\textbf{h}_t$.

In the perspective of SGMs, TimeGrad and ScoreGrad can be regarded as methods to train the following conditional score network $M$:
\begin{align*}
    \sum_{t={t_0}}^T\mathbb{E}_{s}\mathbb{E}_{\textbf{x}_{t}^s} \left\|M_{\theta}(s,\textbf{x}_{t}^s,\textbf{h}_{t-1})-{\nabla}_{\textbf{x}_t^s}\text{log}p(\textbf{x}_t^s|\textbf{h}_{t-1})\right\|_2^2.
\end{align*}
We denote $\textbf{x}_t^s$ as $\textbf{x}_t$ at step $s$. Note that $s$ is uniformly sampled from $[0,1]$ as in ~\cite{song2021SDE}.

\cite{yan2021scoregrad} proposed an energy-based generative method for forecasting. Almost all notations and ideas are the same as those in~\cite{rasul2021timegrad}, except that it generalizes diffusion models into the energy-based field. The used training loss is as follows:
\begin{align*}
    \sum_{t={t_0}}^T\mathbb{E}_{s}\mathbb{E}_{\textbf{x}_{t}}\mathbb{E}_{\textbf{x}_{t}^s|\textbf{x}_{t}}\left\|M_{\theta}(s,\textbf{x}_{t}^s,\textbf{x}_{1:t-1})-{\nabla}_{\textbf{x}_t^s}\text{log}p(\textbf{x}_t^s|\textbf{x}_{t})\right\|_2^2.
\end{align*}

By using a continuous model,~\cite{yan2021scoregrad} achieved good results outperforming~\cite{rasul2021timegrad}. ~\cite{tashiro2021csdi} proposed a general diffusion framework which can be applied to various domains, including not only imputation, but also forecasting and interpolations. CSDI reconstructs an entire sequence at once, not recursively. Although the model takes a diffusion framework, their loss can be written as follows:
\begin{align*}
    \mathbb{E}_{s}\mathbb{E}_{\textbf{x}_{ta}^s}\left\|M_{\theta}(s,\textbf{x}_{ta}^s,\textbf{x}_{co})-{\nabla}\text{log}p(\textbf{x}_{ta}^s|\textbf{x}_{co})\right\|_2^2,
\end{align*}
\noindent where $\textbf{x}_{co}$ and $\textbf{x}_{ta}$ are conditions and imputation targets, respectively. By training a score network using the above loss, the model generates an entire sequence from a partial sequence. 

Although ~\cite{rasul2021timegrad, yan2021scoregrad} earned the state-of-the-art results for forecasting and imputation, we found that they are not suitable for our generative task due to the fundamental mismatch between their score function definitions and our task (cf. Table~\ref{table4}). For instance, TimeGrad generates future observations given the hidden representation of past observations $\textbf{h}_{t-1}$, i.e., a representative forecasting structure.

As a remedy, we propose to optimize a conditional score network by using the following denoising score matching:
\begin{align*}
    \mathbb{E}_{s}\mathbb{E}_{\textbf{h}_{1:T}}\sum_{t=1}^T \left\|M_{\theta}(s,{\textbf{h}}_{t}^s,\textbf{h}_{t-1})-{\nabla}_{\textbf{h}_{t}^s}\text{log}p(\textbf{h}_{t}^s|\textbf{h}_{t})\right\|_2^2.
\end{align*}
We denote $\textbf{h}_0 = \textbf{0}$. Unlike other models~\cite{rasul2021timegrad, yan2021scoregrad, tashiro2021csdi} which resort to existing known proofs, we prove the correctness of our denoising score matching loss in Theorem~\eqref{thm1}.

\section{Proposed Methods : TSGM}

TSGM consists of three networks: an encoder, a decoder, and a conditional score network. Firstly, we train the encoder and the decoder to embed time-series into a latent space. Next, using the pre-trained encoder and decoder, we train the conditional score network. The conditional score network will be used for sampling fake time-series.

\subsection{Encoder and Decoder}
The encoder and decoder map time-series to a latent space and vice versa. Let $\mathcal{X}$ and $\mathcal{H}$ denote a data space and a latent space. Next, $e$ and $d$ are an embedding function mapping $\mathcal{X}$ to $\mathcal{H}$ and vice versa. We define $\mathbf{x}_{1:T}$ as time-series sample with a length of $T$, and $\mathbf{x}_t$ is a multi-dimensional observation in $\mathbf{x}_{1:T}$ at time $t$. Similarly, $\mathbf{h}_{1:T}$ and $\mathbf{h}_t$ are embedded vectors. The encoder $e$ and decoder $d$ are constructed using recurrent neural networks. Since TSGM uses RNNs, both $e$ and $d$ are defined recursively as follows:
\begin{align}
    \mathbf{h}_t = e(\mathbf{h}_{t-1}, \mathbf{x}_t), \qquad \hat{\mathbf{x}}_t = d(\mathbf{h}_t),
\end{align}where $\hat{\mathbf{x}}_t$ denotes a reconstructed time-series sample at time $t$. After embedding real time-series data onto a latent space, we can train the conditional score network with its conditional log likelihood, which will be described in the following subsection. The encoder and decoder are pre-trained before our main training with the proposed denoising score matching.

\subsection{Conditional Score Network}

Unlike other generation tasks, e.g., image generation~\cite{song2021SDE}, tabular data synthesis~\cite{kim2022sos}, where each sample is independent, time-series observations are dependent to previous observations. Therefore, the score network for time-series generation must be designed to learn the conditional log likelihood on previous generated observations, which is more complicated than that in image generation. 

In order to learn the conditional log likelihood, we modify the popular U-net~\cite{ronneberger2015unet} architecture for our purpose. Since U-net has achieved various good results for previous generative tasks~\cite{song2019smld,song2021SDE}, we modify its 2-dimensional convolution layers to 1-dimensional ones for handling time-series observations. Then we concatenate diffused data with condition, and use the concatenated one and temporal feature as input to learn score function. More details are in the training and sampling procedures section.

\begin{figure*}[t]
\centering
\includegraphics[width=0.98\textwidth]{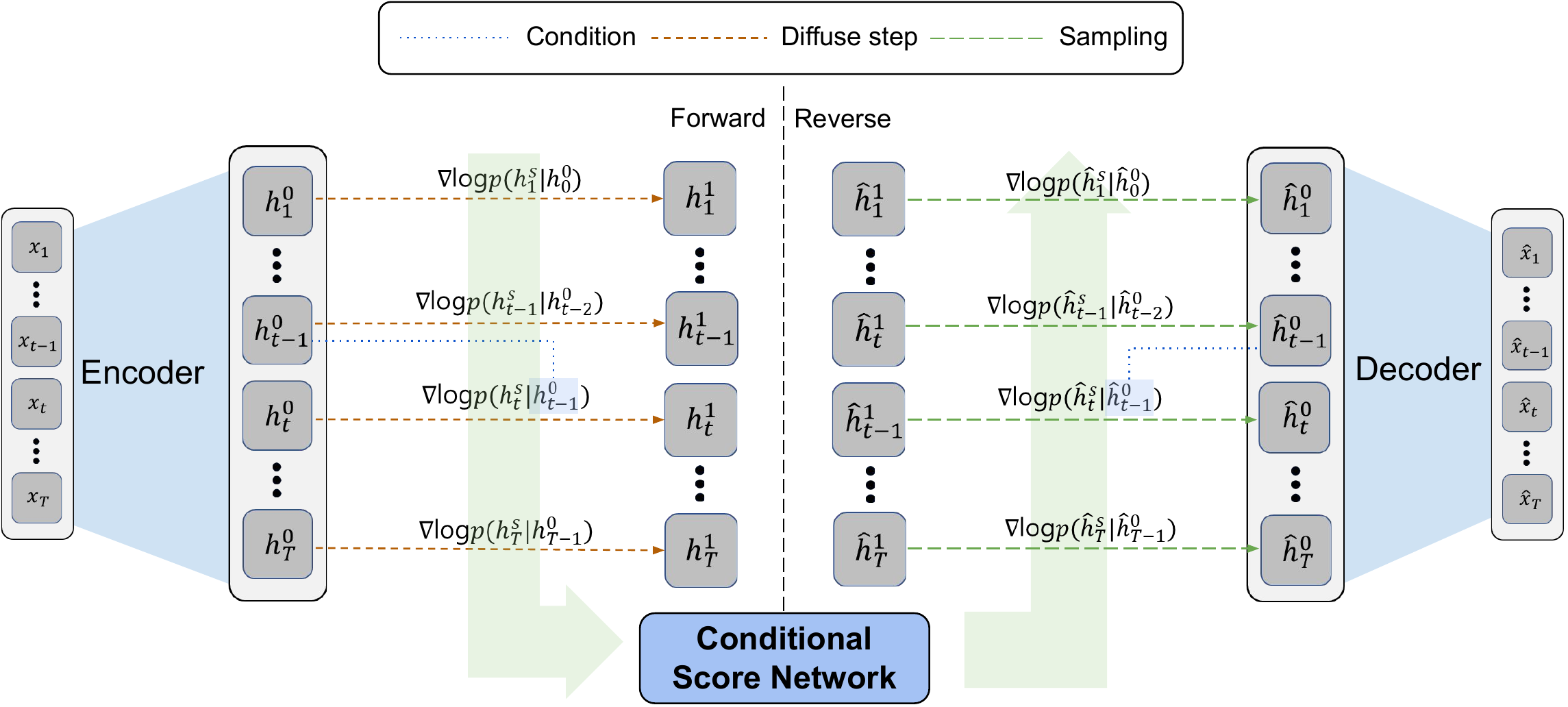}
\caption{The overall workflow of TSGM}
\label{fig1}
\end{figure*}

\subsection{Training Objective Functions}
We use two training objective functions. First, we train the encoder and the decoder using $L_{ed}$. Let $\mathbf{x}_{1:T} \sim p(\mathbf{x}_{1:T})$ and $\hat{\mathbf{x}}_{1:T}$ denote an input time-series sample and its reconstructed copy by the encoder-decoder process, respectively. Then, $L_{ed}$ means the following MSE loss between $\mathbf{x}_{1:T}$ and its reconstructed copy $\hat{\mathbf{x}}_{1:T}$:
\begin{align}
L_{ed} = \mathbb{E}_{\mathbf{x}_{1:T}}[\left\|\hat{\mathbf{x}}_{1:T} - \mathbf{x}_{1:T} \right\|_2^2].
\end{align}

Next, we define another loss $L_{score}$ to train the conditional score network, which is one of our main contributions. At time $t$ in $[1:T]$, we diffuse $\mathbf{x}_{1:t}$ through a sufficiently large number of steps of the forward SDE to a Gaussian distribution. Let $\mathbf{x}_{1:t}^s$ denotes a diffused sample at step $s \in [0,1]$ from $\mathbf{x}_{1:t}$. Then, the conditional score network $M_{\theta}({s, \mathbf{x}}_{1:t}^s, \mathbf{x}_{1:t-1})$ learns the gradient of the conditional log-likelihood with following $L_{1}$: 
\begin{small}
\begin{equation}
L_{1} = \mathbb{E}_{s}\mathbb{E}_{{\textbf{x}}_{1:T}}\left[\sum_{t=1}^{T}\lambda(s)l_{1}(t,s)\right], 
\label{equation4}
\end{equation}
\end{small}where 
\begin{footnotesize}
\begin{equation*}
l_{1}(t,s) = \mathbb{E}_{{\textbf{x}}_{1:t}^s}\left[\left\|M_{\theta}(s,{\textbf{x}}_{1:t}^s,\textbf{x}_{1:t-1})-{\nabla}_{{\textbf{x}}_{1:t}^s}\text{log}p({\textbf{x}}_{1:t}^s|\textbf{x}_{1:t-1})\right\|_2^2\right] 
\end{equation*}
\end{footnotesize}that ${\nabla}_{{\textbf{x}}_{1:t}^s}\text{log}p({\textbf{x}}_{1:t}^s|\textbf{x}_{1:t-1})$, where $\textbf{x}_{i}$ in depends on $\textbf{x}_{1:i-1}$ for each $i \in [2,t]$, is designed specially for time-series generation. Note that for our training, ${\textbf{x}}_{1:t}^s$ is sampled from ${p}({\textbf{\textbf{x}}}_{1:t}^s|\textbf{x}_{1:t-1})$ and $s$ is uniformly sampled from $[0,1]$. Thanks to following theorem, we can use the efficient denoising score loss $L_{2}$ to train the conditional score network. We set $\textbf{x}_0=\textbf{0}$. 

\begin{theorem}[Denoising score matching on time-series]\label{thm1}
$l_{1}(t,s)$ can be replaced by the following $l_{2}(t,s)$:
\begin{footnotesize}
\begin{align*}
l_{2}(t,s) = \mathbb{E}_{\textbf{\textsc{x}}_t}\mathbb{E}_{{\textbf{\textsc{x}}}_{1:t}^s}\left[\left\|M_{\theta}(s,{\textbf{\textsc{x}}}_{1:t}^s,\textbf{\textsc{x}}_{1:t-1})-{\nabla}_{{\textbf{\textsc{x}}}_{1:t}^s}\text{log}p({\textbf{\textsc{x}}}_{1:t}^s|\textbf{\textsc{x}}_{1:t})\right\|_2^2 \right],
\end{align*}\end{footnotesize}where $\textbf{\textsc{x}}_t$ and ${\textbf{\textsc{x}}}_{1:t}^s$ are sampled from ${p}(\textbf{\textsc{x}}_t|\textbf{\textsc{x}}_{1:t-1})$ and ${p}({\textbf{\textsc{x}}}_{1:t}^s|\textbf{\textsc{x}}_{1:t})$. Therefore, we can use an alternative objective, $L_{2} = \mathbb{E}_{s}\mathbb{E}_{{\textbf{\textsc{x}}}_{1:T}}\left[\sum_{t=1}^{T}\lambda(s)l_{2}(t,s)\right]$ instead of $L_{1}$
\end{theorem}
\begin{proof}
We first decompose $l_1(t,s)$ into three parts. We then transform each part by inducing the following result: $\mathbb{E}_{\textbf{x}_{1:t}^s}{\nabla}_{{\textbf{x}}_{1:t}^s}\text{log}p({\textbf{x}}_{1:t}^s|\textbf{x}_{1:t-1})=\mathbb{E}_{\textbf{x}_{t}}\mathbb{E}_{\textbf{x}_{1:t}^s}{\nabla}_{{\textbf{x}}_{1:t}^s}\text{log}p({\textbf{x}}_{1:t}^s|\textbf{x}_{1:t})$. Detailed proof is in Appendix.
\end{proof}

Next, we describe another corollary to formulate our loss.

\begin{corollary} Our target objective function, $L_{score}$, is defined as follows:
\begin{small}
\begin{align*}
L_{score} = \mathbb{E}_{s}\mathbb{E}_{\textbf{x}_{1:T}}\left[\sum_{t=1}^{T}\lambda(s)l^\star_2(t, s) \right],
\end{align*}where
\begin{align*}
l^\star_{2}(t,s) = \mathbb{E}_{{\textbf{\textsc{x}}}_{1:t}^s}\left[\left\|M_{\theta}(s,{\textbf{\textsc{x}}}_{1:t}^s,\textbf{\textsc{x}}_{1:t-1})-{\nabla}_{{\textbf{\textsc{x}}}_{1:t}^s}\text{log}p({\textbf{\textsc{x}}}_{1:t}^s|\textbf{\textsc{x}}_{1:t})\right\|_2^2\right].
\end{align*}
\end{small}Then, $L_{2}=L_{score}$ is satisfied.
\end{corollary}

\noindent The  proof  of  Corollary 1 is  given  in  Appendix. Note that the loss $L_{2}$ cannot be used in practice since it has to sample each $\textbf{x}_t$ using $p(\textbf{x}_t|\textbf{x}_{1:t-1})$ every time. So we use the loss $L_{score}$ for our training.

Since we pre-train the encoder and decoder, the encoder can embeds $\textbf{x}_{1:t}$ into $\textbf{h}_t \in \mathcal{H}$. Ideally, $\textbf{h}_t$ involves the entire information of $\textbf{x}_{1:t}$, but we want to consider more practical cases. Thus, we define $\mathcal{H}$ as a latent space having the past information by setting up window size $K$ as constant sequential length of each sample (details are in the experimental setup section). Therefore, $L_{score}$ can be re-written as follows with the embedding into the latent space:
\begin{equation}
 L_{score}^{\mathcal{H}} = \mathbb{E}_{s}\mathbb{E}_{\textbf{h}_{1:T}}\sum_{t=1}^{T}\left[\lambda(s)l_3(t, s)\right],
\end{equation}
\noindent with {\small $l_{3}(t, s) = \mathbb{E}_{\textbf{h}_{t}^s}\left[\left\|M_{\theta}(s,\textbf{h}_{t}^s,\textbf{h}_{t-1})-{\nabla}_{\textbf{h}_{t}^s}\text{log}p(\textbf{h}_{t}^s|\textbf{h}_{t})\right\|_2^2\right]$}. $L_{score}^{\mathcal{H}}$ is what we use for our experiments (instead of $L_{score}$). Until now, we introduced our target objective functions, $L_{ed}$ and $L_{score}^{\mathcal{H}}$. We note that we use the exactly same weight $\lambda(s)$ as that in~\cite{song2021SDE}. 

\subsection{Training and Sampling Procedures}

We explain the details about its training method with an intuitive example in Fig.~\ref{fig1}. At first, we pre-train both the encoder and decoder using $L_{ed}$. After pre-training them, we train the conditional score network. When training the conditional score network, we use the embedded hidden vectors produced by the encoder. After encoding an input $\mathbf{x}_{1:T}$, we obtain its latent vectors $\mathbf{h}_{1:T}$ --- we note that each hidden vector $\mathbf{h}_t$ has all the previous information on or before $t$ since the encoder is an RNN-based encoder. We use the following forward process ~\cite{song2021SDE}, where $t$ means the physical time of the input time-series in $[1:T]$ and $s$ denotes the time of the diffusion step :
\begin{align*}
    d\textbf{h}_t^s=\textbf{f}(s,\textbf{h}_t^s)ds+g(s)d\textbf{w}, \qquad s \in [0,1].
\end{align*}

During the forward process, the conditional score network reads the pair ($s$, $\mathbf{h}_t^s$, $\mathbf{h}_{t-1}$) as input and thereby, it can learns the conditional score function $\nabla \log p(\mathbf{h}_t^s | \mathbf{h}_{t-1})$ by using $L_{score}^{\mathcal{H}}$, where $\mathbf{h}_{0} = \mathbf{0}$. We train the conditional score network and the encoder-decoder pair alternately after the pre-training step. For some datasets, we found that train only the conditional score network achieve better results after pre-training the autoencoder. Therefore, $use_{alt}=\{True, False\}$ is a hyperparameter to set whether we use the alternating training method. We give the detailed training procedure in Algorithm~\ref{algorithm1}

\begin{algorithm}[tb]
\caption{Training algorithm}
\label{algorithm1}
\textbf{Input}: $\mathbf{x}_{1:T}$\\
\textbf{Parameter}: \\
$use_{alt}$ = A Boolean parameter to set whether to use the alternating training method.\\
$iter_{pre}$ = The number of iterations for pre-training\\
$iter_{main}$ = The number of iterations for training\\
\textbf{Output}: $Encoder, Decoder, M_{\theta}$
\begin{algorithmic}[1] 
\FOR{$iter \in \{1,...,iter_{pre}\}$}
\STATE Train $Encoder$ and $Decoder$ by using $L_{ed}$ 
\ENDFOR
\FOR{$iter \in \{1,...,iter_{main}\}$}
\STATE Train $M_{\theta}$ by using $L_{score}^{\mathcal{H}}$
\IF{$use_{alt}$}
\STATE Train the $Encoder$ and $Decoder$ by using $L_{ed}$
\ENDIF
\ENDFOR
\STATE \textbf{return} $Encoder, Decoder, M_{\theta}$
\end{algorithmic}
\end{algorithm}

After the training procedure, we use the following conditional reverse process:
\begin{align*}
    d\textbf{h}_t^s=[\textbf{f}(s,\textbf{h}_t^s)-g^2(s)\nabla_{\textbf{h}_t^s}{\text{log}p(\textbf{h}_t^s|\textbf{h}_{t-1})}]ds+g(s)d\bar{\textbf{w}},
\end{align*}where $s$ is uniformly sampled over $[0,1]$. The conditional score function in this process can be replaced with the trained score network $M_{\theta}(s,\textbf{h}_{t}^s,\textbf{h}_{t-1})$. The detailed sampling method is as follows:
\begin{itemize}
    \item At first, we sample $\textbf{z}_0$ from a Gaussian prior distribution and concatenate it with $\mathbf{h}_{0}$. 
    We consider the concatenated vector and temporal feature as inputs to the conditional score network and generates initial data $\hat{\textbf{h}}_1$ with the \textit{predictor-corrector} method~\cite{song2021SDE}.

    \item We repeat the following computation for every $1 < t \leq T$, i.e., recursive generation. After reading the previously generated samples $\hat{\textbf{h}}_{t-1}$ for $t \in [2,T]$, we sample $\textbf{z}_{t-1}$ from a Gaussian prior distribution and concatenate it with $\hat{\textbf{h}}_{t-1}$. We use it and temporal feature, $s$, as inputs to the conditional score network $M_{\theta}(s,\textbf{h}_{t}^s,\textbf{h}_{t-1})$ and synthesize the next latent vector $\hat{\textbf{h}}_{t}$ via the \textit{predictor-corrector} step.
\end{itemize}
  
\noindent Once the sampling procedure is finished, we can reconstruct ${\hat{\textbf{x}}}_{1:T}$ from $\hat{\textbf{h}}_{1:T}$ using the trained decoder at once.

\section{Experiments}
We describe our detailed experimental environments and results.

\subsection{Experimental Environments}

We use 5 real-world datasets from various fields with 8 baselines. We refer to Appendix for the detailed descriptions on our datasets, baselines, and other software/hardware environments. In particular, our collection of baselines covers almost all existing types of time series synthesis methods, ranging from RNNs to VAEs and GANs. For the baselines, we reuse their released source codes in their official repositories and rely on their designed training and model selection procedures.

\subsubsection{Hyperparameters}

Table~\ref{table6} shows the best hyperparameters. We define hidden dimension of inputs and temporal features as $d_{in}$ and $d_t$, respectively. For other settings, we use the default values in TimeGAN~\cite{yoon2019timegan} and VPSDE~\cite{song2021SDE}. We give the detailed search method in Appendix.

\begin{table}[h]
\centering
\begin{tabular}{c|c|c|c|c|c}
    \hline
    Dataset & $d_{in}$ & $d_{t}$ & $use_{alt}$ & $iter_{pre}$ & $iter_{main}$ \\
    \hline
    Stocks & 24 & 96 & True & 50000 & \multirow{5}{*}{40000}\\
    Energy & 56 & 56 & False & 100000 &\\
    Air & 40 & 80 & True & 50000 &\\
    AI4I & 24 & 96 & True & 50000 &\\
    Occupancy & 40 & 80 & False & 100000 &\\
    \hline
\end{tabular}
\caption{The best hyperparameters for our method}
\label{table6}
\end{table}

\begin{table*}[t]
\centering

\begin{tabular}{c|c|c|c|c|c|c}
    \hline
    & Dataset & Stocks & Energy & Air & AI4I & Occupancy\\
    \hline
    \parbox[t]{2mm}{\multirow{11}{*}{\rotatebox{90}{Disc. score}}} & TSGM-VP & .022$\pm$.005 & .221$\pm$.025 & \textbf{.122}$\pm$\textbf{.014} & .147$\pm$.005 & .402$\pm$.004 \\
    & TSGM-subVP & \textbf{.021}$\pm$\textbf{.008} & \textbf{.198}$\pm$\textbf{.025} & .127$\pm$.010 & .150$\pm$.010 & .414$\pm$.008 \\
    & T-Forcing & .226$\pm$.035 & .483$\pm$.004 & .404$\pm$.020 & .435$\pm$.025 & .333$\pm$.005 \\
    & P-Forcing & .257$\pm$.026 & .412$\pm$.006 & .484$\pm$.007 & .443$\pm$.026 & .411$\pm$.013 \\
    & TimeGAN & .102$\pm$.031 & .236$\pm$.012 & .447$\pm$.017 & \textbf{.070}$\pm$\textbf{.009} & .365$\pm$.014 \\
    & RCGAN & .196$\pm$.027 & .336$\pm$.017 & .459$\pm$.104 & .234$\pm$.015 & .485$\pm$.001 \\
    & C-RNN-GAN & .399$\pm$.028 & .499$\pm$.001 & .499$\pm$.000 & .499$\pm$.001 & .467$\pm$.009 \\
    & TimeVAE & .175$\pm$.031 & .498$\pm$.006 & .381$\pm$.037 & .446$\pm$.024 & .415$\pm$.050 \\
    & WaveGAN & .217$\pm$.022 & .363$\pm$.012 & .491$\pm$.013 & .481$\pm$.034 & \textbf{.309}$\pm$\textbf{.039} \\
    & COT-GAN & .285$\pm$.030 & .498$\pm$.000 & .423$\pm$.001 & .411$\pm$.018 & .443$\pm$.014 \\
    \hline
    \parbox[t]{2mm}{\multirow{12}{*}{\rotatebox{90}{Pred. score}}} & TSGM-VP & .037$\pm$.000 & .257$\pm$.000 & \textbf{.005}$\pm$\textbf{.000} & .218$\pm$.000 & .022$\pm$.001 \\
    & TSGM-subVP & \textbf{.037}$\pm$\textbf{.000} & \textbf{.252}$\pm$\textbf{.000} & .005$\pm$.000 & \textbf{.217}$\pm$\textbf{.000} & \textbf{.022}$\pm$\textbf{.001} \\
    & T-Forcing & .038$\pm$.001 & .315$\pm$.005 & .008$\pm$.000 & .242$\pm$.001 & .029$\pm$.001 \\
    & P-Forcing & .043$\pm$.001 & .303$\pm$.006 & .021$\pm$.000 & .220$\pm$.000 & .070$\pm$.105 \\
    & TimeGAN & .038$\pm$.001 & .273$\pm$.004 & .017$\pm$.004 & .253$\pm$.002 & .057$\pm$.001 \\
    & RCGAN & .040$\pm$.001 & .292$\pm$.005 & .043$\pm$.000 & .224$\pm$.001 & .471$\pm$.002 \\
    & C-RNN-GAN & .038$\pm$.000 & .483$\pm$.005 & .111$\pm$.000 & .340$\pm$.006 & .242$\pm$.001 \\
    & TimeVAE & .042$\pm$.002 & .268$\pm$.004 & .013$\pm$.002 & .233$\pm$.010 & .035$\pm$.002 \\
    & WaveGAN & .041$\pm$.001 & .307$\pm$.007 & .009$\pm$.000 & .225$\pm$.006 & .034$\pm$.008 \\
    & COT-GAN & .044$\pm$.000 & .260$\pm$.000 & .024$\pm$.001 & .220$\pm$.000 & .084$\pm$.001 \\
    \cline{2-7}
    & Original & .036$\pm$.001 & .250$\pm$.003 & .004$\pm$.000 & .217$\pm$.000 & .019$\pm$.000 \\
    \hline
\end{tabular}
\caption{Experimental results in terms of the discriminative and predictive scores. The best scores are in boldface. The statistical significance ($p<0.05$) of the best results are confirmed in all cases by using the Wilcoxon rank sum test..}
\label{table2}
\end{table*}

\subsubsection{Evaluation Metrics}

In the image generation domain, researchers have evaluated the \textit{fidelity} and the \textit{diversity} of models by using the Fr\'echet inception distance (FID) and inception score (IS). On the other hand, to measure the fidelity and the diversity of synthesized time-series samples, we use the following predictive score and the discriminative score as in~\cite{yoon2019timegan}. We strictly follow the agreed evaluation protocol of~\cite{yoon2019timegan} by the time-series research community.  Both metrics are designed in a way that lower values are preferred. We run each generative method 10 times with different seeds, and report its mean and standard deviation of the following discriminative and predictive scores:

i) \textit{Predictive Score}: 
We use the predictive score to evaluate whether a generative model can successfully reproduce the temporal properties of the original data. To do this, we first train a naive LSTM-based sequence model for time-series forecasting with synthesized samples. The performance of this predictive model is measured as the mean absolute error (MAE) on the original test data. This kind of evaluation paradigm is called as train-synthesized-test-real (TSTR) in the literature.

ii) \textit{Discriminative Score}:
In order to assess how similar the original and generated samples are, we train a 2-layer LSTM model that classifies the real/fake samples into two classes, real or fake. We use the performance of the trained classifier on the test data as the discriminative score. Therefore, lower discriminator scores mean that real and fake samples are similar.

\begin{figure*}[t]
\centering
\includegraphics[width=0.9\textwidth]{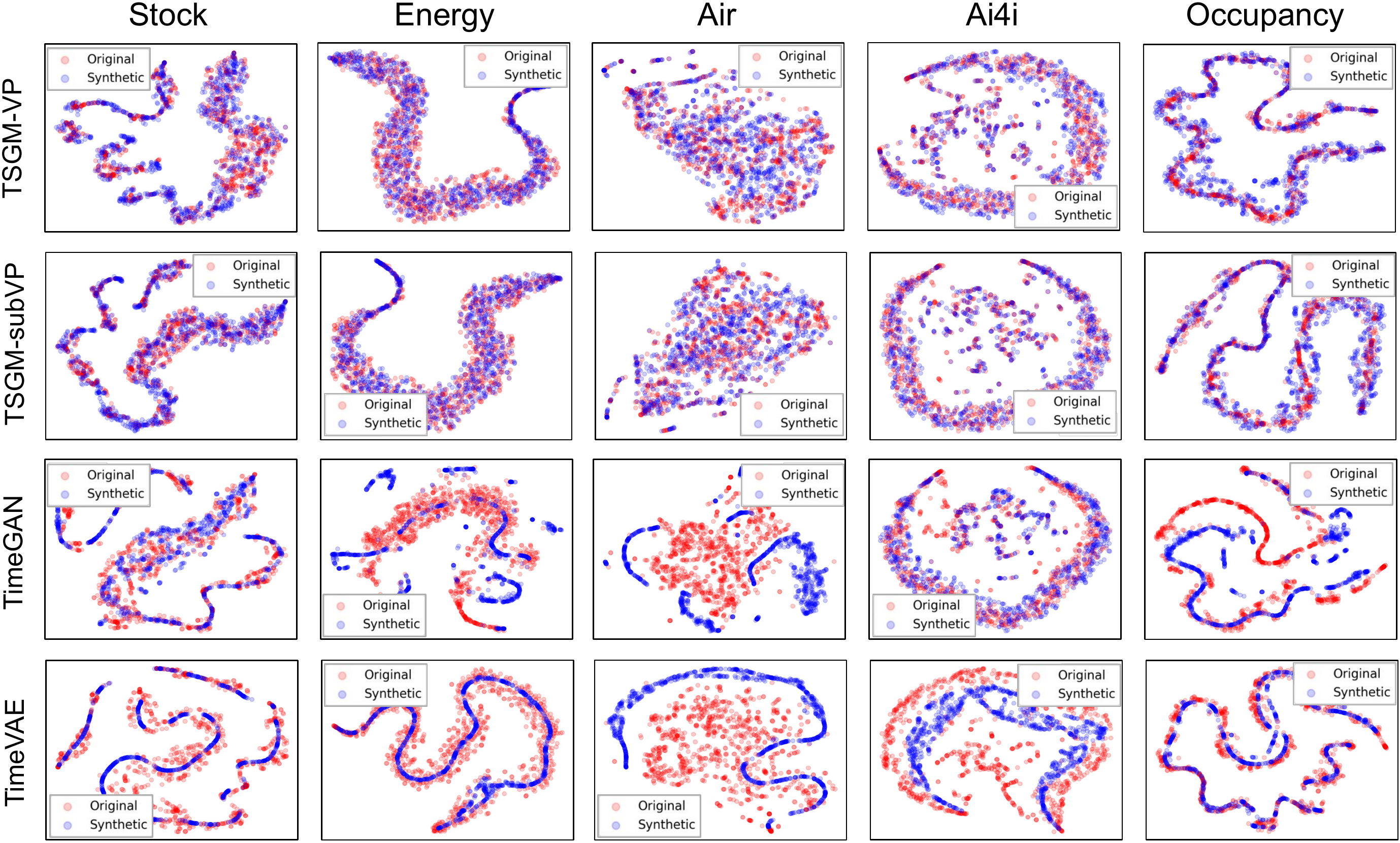} 
\caption{t-SNE plots for TSGM (1st and 2nd rows), TimeGAN (3rd row), and TimeVAE (4th row). Red and blue dots mean original and synthesized samples, respectively.}
\label{fig2}
\end{figure*}

\subsection{Experimental Results}
Table~\ref{table2} shows that achieves remarkable results, outperforming TimeGAN, a state-of-the-art method, except for only two cases: the discriminative scores on AI4I and Occupancy. Especially, for Stock, Energy, and Air, TSGM exhibits overwhelming performance by large margins for the discriminative score.
Moreover, for the predictive score, TSGM performs best and obtains almost the same scores with that of the original data, which indicates that generated samples from TSGM preserve almost all the predictive characteristics of the original data.

We also show t-SNE visualizations in Figure~\ref{fig2}. TimeGAN and TimeVAE are two representative VAE and GAN-based baselines, respectively. In the figure, the synthetic samples generated from TSGM consistently show successful recall from the original data. Therefore, TSGM generates diverse synthetic samples in comparison with TimeGAN and TimeVAE in all cases. Especially, TSGM achieves much higher diversity than the two models on Energy and Air.

\subsection{Ablation and Sensitivity Studies}

\subsubsection{Ablation study}
As an ablation study, we simultaneously train the conditional score network, encoder, and decoder from scratch. The results are in Table~\ref{table3}. There ablation models are worse than the full model, but they still outperform many baselines.

\begin{table}[t]
\centering
{\small
\begin{tabular}{c|c|c|c|c}
    \hline
    & Model & SDE & Stocks & Energy \\
    \hline
    \parbox[t]{2mm}{\multirow{4}{*}{\rotatebox{90}{Disc.}}} & \multirow{2}{*}{TSGM} & VP & .022$\pm$.005 & .221$\pm$.025 \\
    & & subVP & \textbf{.021}$\pm$\textbf{.008} & \textbf{.198}$\pm$\textbf{.025} \\
    \cline{2-5}
    & {\multirow{2}{*}{w/o pre-training}} & VP & .022$\pm$.004 & .322$\pm$.003 \\
    & & subVP & .059$\pm$.006 & .284$\pm$.004  \\
    \hline
    \parbox[t]{2mm}{\multirow{4}{*}{\rotatebox{90}{Pred.}}} & \multirow{2}{*}{TSGM} & VP & .037$\pm$.000  & .257$\pm$.000 \\
    & & subVP & \textbf{.037}$\pm$\textbf{.000} & .252$\pm$.000 \\
    \cline{2-5}
    & \multirow{2}{*}{w/o pre-training} & VP & .037$\pm$.000 & .252$\pm$.000 \\
    & & subVP & .037$\pm$.000 & \textbf{.251}$\pm$\textbf{.000} \\
    \hline
\end{tabular}
}
\caption{Comparison between with and without pre-training the autoencoder. For other omitted datasets, we can observe similar patterns.}
\label{table3}
\end{table}

\subsubsection{Sensitivity study} We conduct two sensitivity studies: i) reducing the depth of U-net, ii) decreasing the sampling step numbers. The results are in Table~\ref{table5}.

At first, we modify the depth of U-net from 4 to 3 to check the performance of the lighter conditional score network. Surprisingly, we achieve a better discriminative score with a slight loss on the predictive score.

Next, we decrease the number of sampling steps for faster sampling from 1000 steps to 500, 250, and 100 steps, respectively. For VP, we achieve almost the same results with 500 steps. Surprisingly, in the case of subVP, we achieve good results until 100 steps.

\subsection{Empirical Complexity Analyses} We also report the memory footprint and runtime for sampling fake data in Appendix for space reasons. In general, our method requires a longer sampling time than TimeGAN and some other simple methods.

\section{Conclusion \& Limitations}

We presented a score-based generative model framework for time-series generation. We combined an RNN-based autoencoder and our score network into a single framework to accomplish the goal. We also designed an appropriate denoising score matching loss for our generation task and achieved the state-of-the-art results on various datasets in terms of the discriminative and predictive scores. In addition, we conducted rigorous ablation and sensitivity studies to prove the efficacy of our model design.

Although our method achieves the state-of-the-art sampling quality and diversity, there exists a fundamental problem that all SGMs have. That is, SGMs are slower than GANs for generating samples. Since there are several accomplishments for faster sampling~\cite{xiao2022tackling,DBLP:journals/corr/abs-2105-14080}, however, one can apply them to our method and it would be much faster without any loss of sampling quality and diversity. One can also use other types of autoencoders instead of our RNN-based one. For instance, other continuous-time autoencoder methods are known to be able to process irregular time-series~\cite{NIPS2019_8773}. Our proposed score-based methods are naturally benefited by using those advanced autoencoders. In addition, we did not search miscellaneous hyperparameters by reusing the default values in TimeGAN~\cite{yoon2019timegan} and VPSDE~\cite{song2021SDE}, which means there still exists room to improve.

\begin{table}[h]
\centering
{\small
\begin{tabular}{c|c|c|c|c}
    \hline
    & Model & SDE & Stocks & Energy \\
    \hline
    \parbox[t]{2mm}{\multirow{10}{*}{\rotatebox{90}{Disc.}}} & \multirow{2}{*}{TSGM} & VP & .022$\pm$.005 & .221$\pm$.025 \\
    & & subVP & .021$\pm$.008 & .198$\pm$.025 \\
    \cline{2-5}
    & {\multirow{2}{*}{depth of 3}} & VP & .022$\pm$.004 & \textbf{.175}$\pm$\textbf{.009} \\
    & & subVP & \textbf{.020}$\pm$\textbf{.007} & .182$\pm$.009  \\
    \cline{2-5}
    & {\multirow{2}{*}{500 steps}} & VP & .025$\pm$.005 & .259$\pm$.003 \\
    & & subVP & .020$\pm$.004 & .248$\pm$.002  \\
    \cline{3-5}
    & {\multirow{2}{*}{250 steps}} & VP & .067$\pm$.009 & .250$\pm$.003 \\
    & & subVP & .022$\pm$.009 & .247$\pm$.002 \\
    \cline{3-5}
    & {\multirow{2}{*}{100 steps}} & VP & .202$\pm$.013 & .325$\pm$.003 \\
    & & subVP & .023$\pm$.005 & .237$\pm$.004  \\
    \hline
    \parbox[t]{2mm}{\multirow{10}{*}{\rotatebox{90}{Pred.}}} & \multirow{2}{*}{TSGM} & VP &.037$\pm$.000 & .257$\pm$.000 \\
    & & subVP &\textbf{.037}$\pm$\textbf{.000} & \textbf{.252}$\pm$\textbf{.000} \\
    \cline{2-5}
    & {\multirow{2}{*}{depth of 3}} & VP & .037$\pm$.000 & .253$\pm$.000 \\
    & & subVP &.037$\pm$.000 & .253$\pm$.000  \\
    \cline{2-5}
    & {\multirow{2}{*}{500 steps}} & VP & .037$\pm$.000 & .257$\pm$.000 \\
    & & subVP & .037$\pm$.000 & .253$\pm$.000 \\
    \cline{3-5}
    & {\multirow{2}{*}{250 steps}} & VP & .037$\pm$.000 & .256$\pm$.000 \\
    & & subVP & .037$\pm$.000 & .253$\pm$.000  \\
    \cline{3-5}
    & {\multirow{2}{*}{100 steps}} & VP & .039$\pm$.000 & .256$\pm$.000 \\
    & & subVP & .037$\pm$.000 & .253$\pm$.000  \\
    \hline
\end{tabular}
}
\caption{Sensitivity results on the depth of U-net and the number of sampling steps}
\label{table5}
\end{table}


\bibliography{aaai23}

\clearpage
\newpage
\renewcommand{\thesubsection}{\Alph{subsection}}
\setcounter{secnumdepth}{2}
\setcounter{table}{0}
\renewcommand{\thetable}{\Alph{table}}
\setcounter{figure}{0}
\renewcommand{\thefigure}{\Alph{figure}}
\setcounter{equation}{0}
\renewcommand{\theequation}{\alph{equation}}

\onecolumn
\twocolumn[
\centering
\Large
\textbf{TSGM: Time-series Generation using Score-based Generative Models} \\
\vspace{0.5em}-- Supplementary Material -- \\
\vspace{1.0em}
]
\appendix

\subsection{Proofs}

\begin{theorem}[Denoising score matching on time-series]
$l_{1}(t,s)$ can be replaced by the following $l_{2}(t,s)$:
\begin{footnotesize}
\begin{align*}
l_{2}(t,s) = \mathbb{E}_{\textbf{\textsc{x}}_t}\mathbb{E}_{{\textbf{\textsc{x}}}_{1:t}^s}\left[\left\|M_{\theta}(s,{\textbf{\textsc{x}}}_{1:t}^s,\textbf{\textsc{x}}_{1:t-1})-{\nabla}_{{\textbf{\textsc{x}}}_{1:t}^s}\text{log}p({\textbf{\textsc{x}}}_{1:t}^s|\textbf{\textsc{x}}_{1:t})\right\|_2^2 \right],
\end{align*}\end{footnotesize}where $\textbf{\textsc{x}}_t$ and ${\textbf{\textsc{x}}}_{1:t}^s$ are sampled from ${p}(\textbf{\textsc{x}}_t|\textbf{\textsc{x}}_{1:t-1})$ and ${p}({\textbf{\textsc{x}}}_{1:t}^s|\textbf{\textsc{x}}_{1:t})$. Therefore, we can use an alternative objective, $L_{2} = \mathbb{E}_{s}\mathbb{E}_{{\textbf{\textsc{x}}}_{1:T}}\left[\sum_{t=1}^{T}\lambda(s)l_{2}(t,s)\right]$ instead of $L_{1}$
\end{theorem}

\textit{proof}. At first, if $t=1$, it can be substituted with the naive denoising score loss by~\cite{vincent2011matching} since $\textbf{x}_0 = \textbf{0}$. 

Next, let us consider $t>1$. $l_{1}(t,s)$ can be decomposed as follows:
\begin{footnotesize}\begin{align*}
&l_{1}(t,s) = -2\cdot\mathbb{E}_{{\textbf{x}}_{1:t}^s}\langle M_{\theta}(s,{\textbf{x}}_{1:t}^s,\textbf{x}_{1:t-1}),{\nabla}_{{\textbf{x}}_{1:t}^s}\text{log}p({\textbf{x}}_{1:t}^s|\textbf{x}_{1:t-1})\rangle\\
&+\mathbb{E}_{{\textbf{x}}_{1:t}^s}\left[\left\|M_{\theta}(s,{\textbf{x}}_{1:t}^s,\textbf{x}_{1:t-1})\right\|_2^2\right]+C_1
\end{align*}\end{footnotesize}

Here, $C_1$ is a constant that does not depend on the parameter $\theta$, and $\langle\cdot,\cdot\rangle$ mean the inner product. Then, the first part's expectation of the right-hand side can be expressed as follows:
\begin{scriptsize}
\begin{align*}
&\mathbb{E}_{{\textbf{x}}_{1:t}^s}[\langle M_{\theta}(s,{\textbf{x}}_{1:t}^s,\textbf{x}_{1:t-1}),{\nabla}_{{\textbf{x}}_{1:t}^s}\text{log}p({\textbf{x}}_{1:t}^s|\textbf{x}_{1:t-1})\rangle ]\\
&=\int_{{\textbf{x}}_{1:t}^s}\langle M_{\theta}(s,{\textbf{x}}_{1:t}^s,\textbf{x}_{1:t-1}),{\nabla}_{{\textbf{x}}_{1:t}^s}\text{log}p({\textbf{x}}_{1:t}^s|\textbf{x}_{1:t-1})\rangle\text{p}({\textbf{x}}_{1:t}^s|\textbf{x}_{1:t-1})d{\textbf{x}}_{1:t}^s\\
&=\int_{{\textbf{x}}_{1:t}^s}\langle M_{\theta}(s,{\textbf{x}}_{1:t}^s,\textbf{x}_{1:t-1}),\frac{1}{\text{p}(\textbf{x}_{1:t-1})}{\partial\text{p}({\textbf{x}}_{1:t}^s,\textbf{x}_{1:t-1})\over\partial {\textbf{x}}_{1:t}^s}\rangle d{\textbf{x}}_{1:t}^s\\
&=\int_{{\textbf{x}}_{t}}\int_{{\textbf{x}}_{1:t}^s}\langle M_{\theta}(s,{\textbf{x}}_{1:t}^s,\textbf{x}_{1:t-1}),\frac{1}{\text{p}(\textbf{x}_{1:t-1})}{\partial\text{p}({\textbf{x}}_{1:t}^s,\textbf{x}_{1:t-1},\textbf{x}_t)\over\partial {\textbf{x}}_{1:t}^s}\rangle d{\textbf{x}}_{1:t}^sd\textbf{x}_t\\
&=\int_{{\textbf{x}}_{t}}\int_{{\textbf{x}}_{1:t}^s}\langle M_{\theta}(s,{\textbf{x}}_{1:t}^s,\textbf{x}_{1:t-1}),{\partial\text{p}({\textbf{x}}_{1:t}^s|\textbf{x}_{1:t}))\over\partial {\textbf{x}}_{1:t}^s}\rangle\frac{\text{p}(\textbf{x}_{1:t-1},\textbf{x}_t)}{\text{p}(\textbf{x}_{1:t-1})}d{\textbf{x}}_{1:t}^sd\textbf{x}_t\\
&=\int_{{\textbf{x}}_{t}}\int_{{\textbf{x}}_{1:t}^s}\langle M_{\theta}(s,{\textbf{x}}_{1:t}^s,\textbf{x}_{1:t-1}),{\partial\text{p}({\textbf{x}}_{1:t}^s|\textbf{x}_{1:t})\over\partial {\textbf{x}}_{1:t}^s}\rangle\text{p}(\textbf{x}_t|\textbf{x}_{1:t-1})d{\textbf{x}}_{1:t}^sd\textbf{x}_t\\
&=\mathbb{E}_{\textbf{x}_t}\left[\int_{{\textbf{x}}_{1:t}^s}\langle M_{\theta}(s,{\textbf{x}}_{1:t}^s,\textbf{x}_{1:t-1}),{\partial\text{p}({\textbf{x}}_{1:t}^s|\textbf{x}_{1:t})\over\partial {\textbf{x}}_{1:t}^s}\rangle d{\textbf{x}}_{1:t}^s\right]\\
&=\mathbb{E}_{\textbf{x}_t}\left[\int_{{\textbf{x}}_{1:t}^s}\langle M_{\theta}(s,{\textbf{x}}_{1:t}^s,\textbf{x}_{1:t-1}),{\nabla}_{{\textbf{x}}_{1:t}^s}\text{log}p({\textbf{x}}_{1:t}^s|\textbf{x}_{1:t})\rangle\text{p}({\textbf{x}}_{1:t}^s|\textbf{x}_{1:t})d{\textbf{x}}_{1:t}^s\right]\\
&=\mathbb{E}_{\textbf{x}_t}\mathbb{E}_{{\textbf{x}}_{1:t}^s}[\langle M_{\theta}(s,{\textbf{x}}_{1:t}^s,\textbf{x}_{1:t-1}),{\nabla}_{{\textbf{x}}_{1:t}^s}\text{log}p({\textbf{x}}_{1:t}^s|\textbf{x}_{1:t})\rangle]
\end{align*}
\end{scriptsize}



Similarly, the second part's expectation of the right-hand side can be rewritten as follows:
\begin{scriptsize}\begin{align*}
&\mathbb{E}_{{\textbf{x}}_{1:t}^s}[\left\|M_{\theta}(s,{\textbf{x}}_{1:t}^s,\textbf{x}_{1:t-1})\right\|_2^2]\\
&=\int_{{\textbf{x}}_{1:t}^s}\left\|M_{\theta}(s,{\textbf{x}}_{1:t}^s,\textbf{x}_{1:t-1})\right\|_2^2\cdot\text{p}({\textbf{x}}_{1:t}^s|\textbf{x}_{1:t-1})d{\textbf{x}}_{1:t}^s\\
&=\int_{\textbf{x}_t}\int_{{\textbf{x}}_{1:t}^s}\left\|M_{\theta}(s,{\textbf{x}}_{1:t}^s,\textbf{x}_{1:t-1})\right\|_2^2\cdot\frac{\text{p}({\textbf{x}}_{1:t}^s,\textbf{x}_{1:t-1},\textbf{x}_t)}{\text{p}(\textbf{x}_{1:t-1})}d{\textbf{x}}_{1:t}^sd\textbf{x}_t\\
&=\int_{\textbf{x}_t}\int_{{\textbf{x}}_{1:t}^s}\left\|M_{\theta}(s,{\textbf{x}}_{1:t}^s,\textbf{x}_{1:t-1})\right\|_2^2\cdot\text{p}({\textbf{x}}_{1:t}^s|\textbf{x}_{1:t})\frac{\text{p}(\textbf{x}_{1:t-1},\textbf{x}_t)}{\text{p}(\textbf{x}_{1:t-1})}d{\textbf{x}}_{1:t}^sd\textbf{x}_t\\
&=\int_{\textbf{x}_t}\int_{{\textbf{x}}_{1:t}^s}\left\|M_{\theta}(s,{\textbf{x}}_{1:t}^s,\textbf{x}_{1:t-1})\right\|_2^2\cdot\text{p}({\textbf{x}}_{1:t}^s|\textbf{x}_{1:t})\text{p}(\textbf{x}_t|\textbf{x}_{1:t-1})d{\textbf{x}}_{1:t}^sd\textbf{x}_t\\
&=\mathbb{E}_{{\textbf{x}}_{t}}\mathbb{E}_{{\textbf{x}}_{1:t}^s}[\left\|M_{\theta}(s,{\textbf{x}}_{1:t}^s,\textbf{x}_{1:t-1})\right\|_2^2]
\end{align*}\end{scriptsize}

Finally, by using above results, we can derive following result:
\begin{footnotesize}\begin{align*}
&l_{1} = \mathbb{E}_{{\textbf{x}}_{t}}\mathbb{E}_{{\textbf{x}}_{1:t}^s}\left[\left\|M_{\theta}(s,{\textbf{x}}_{1:t}^s,\textbf{x}_{1:t-1})\right\|_2^2\right]+C_1\\
&-2\cdot\mathbb{E}_{{\textbf{x}}_{t}}\mathbb{E}_{{\textbf{x}}_{1:t}^s}\langle M_{\theta}(s,{\textbf{x}}_{1:t}^s,\textbf{x}_{1:t-1}),{\nabla}_{{\textbf{x}}_{1:t}^s}\text{log}p({\textbf{x}}_{1:t}^s|\textbf{x}_{1:t})\rangle\\
&=\mathbb{E}_{\textbf{x}_t}\mathbb{E}_{{\textbf{x}}_{1:t}^s}\left[\left\|M_{\theta}(s,{\textbf{x}}_{1:t}^s,\textbf{x}_{1:t-1})-{\nabla}_{{\textbf{x}}_{1:t}^s}\text{log}p({\textbf{x}}_{1:t}^s|\textbf{x}_{1:t})\right\|_2^2 \right]+C
\end{align*}\end{footnotesize}

$C$ is a constant that does not depend on the parameter $\theta$. \hfill $\square$

\begin{corollary} Our target objective function, $L_{score}$, is defined as follows:
\begin{small}
\begin{align*}
L_{score} = \mathbb{E}_{s}\mathbb{E}_{\textbf{x}_{1:T}}\left[\sum_{t=1}^{T}\lambda(s)l^\star_2(t, s) \right],
\end{align*}where
\begin{align*}
l^\star_{2}(t,s) = \mathbb{E}_{{\textbf{\textsc{x}}}_{1:t}^s}\left[\left\|M_{\theta}(s,{\textbf{\textsc{x}}}_{1:t}^s,\textbf{\textsc{x}}_{1:t-1})-{\nabla}_{{\textbf{\textsc{x}}}_{1:t}^s}\text{log}p({\textbf{\textsc{x}}}_{1:t}^s|\textbf{\textsc{x}}_{1:t})\right\|_2^2\right].
\end{align*}
\end{small}Then, $L_{2}=L_{score}$ is satisfied.
\end{corollary}

\textit{proof.} Whereas one can use the law of total expectation, which means \textit{$E[X]=E[E[X|Y]]$ if X,Y are on an identical probability space} to show the above formula, we calculate directly. At first, let us simplify the expectation of the inner part with a symbol $f(\textbf{x}_{1:t})$ for our computational convenience, i.e.,
$f(\textbf{x}_{1:t})=\mathbb{E}_{s}\mathbb{E}_{{\textbf{x}}_{1:t}^s}\left[\lambda(s)\left\|M_{\theta}(s,{\textbf{x}}_{1:t}^s,\textbf{x}_{1:t-1})-{\nabla}_{{\textbf{x}}_{1:t}^s}\text{log}p({\textbf{x}}_{1:t}^s|\textbf{x}_{1:t})\right\|_2^2\right]$. Then we have the following definition:
\begin{scriptsize}\begin{align*}
&L_{2} = \mathbb{E}_{s}\mathbb{E}_{\textbf{x}_{1:T}}\left[l_2\right]
=\mathbb{E}_{\textbf{x}_{1:T}}\left[\sum_{t=1}^{T}\mathbb{E}_{{\textbf{x}}_{t}}[f(\textbf{x}_{1:t})]\right]=\sum_{t=1}^{T}\mathbb{E}_{\textbf{x}_{1:T}}\mathbb{E}_{{\textbf{x}}_{t}}[f(\textbf{x}_{1:t})]
\end{align*}\end{scriptsize}

At last, the expectation part can be further simplified as follows:
\begin{scriptsize}\begin{align*}
&\mathbb{E}_{\textbf{x}_{1:T}}\mathbb{E}_{{\textbf{x}}_{t}}[f(\textbf{x}_{1:t})]\\
&=\int_{\textbf{x}_{1:T}}\int_{\textbf{x}_t}f(\textbf{x}_{1:t})p(\textbf{x}_{t}|\textbf{x}_{1:t-1})d\textbf{x}_t\cdot p(\textbf{x}_{1:t-1})p(\textbf{x}_{t:T}|\textbf{x}_{1:t-1})d\textbf{x}_{1:T}\\
&=\int_{\textbf{x}_{1:T}}\int_{\textbf{x}_t}f(\textbf{x}_{1:t})p(\textbf{x}_{1:t})d\textbf{x}_t\cdot p(\textbf{x}_{t:T}|\textbf{x}_{1:t-1})d\textbf{x}_{1:T}\\
&=\int_{\textbf{x}_{t:T}}\left({\int_{\textbf{x}_{1:t}}}f(\textbf{x}_{1:t})p(\textbf{x}_{1:t})d\textbf{x}_{1:t}\right)p(\textbf{x}_{t:T}|\textbf{x}_{1:t-1})d\textbf{x}_{t:T}\\
&=\int_{\textbf{x}_{1:t}}f(\textbf{x}_{1:t})p(\textbf{x}_{1:t})d\textbf{x}_{1:t}\\
&=\int_{\textbf{x}_{1:t}}\left(\int_{\textbf{x}_{t+1:T}}p(\textbf{x}_{t+1:T}|\textbf{x}_{1:t})d\textbf{x}_{t+1:T}\right)f(\textbf{x}_{1:t})p(\textbf{x}_{1:t})d\textbf{x}_{1:t}\\
&=\int_{\textbf{x}_{1:T}}f(\textbf{x}_{1:t})p(\textbf{x}_{1:T})d\textbf{x}_{1:T}\\
&=\mathbb{E}_{\textbf{x}_{1:T}}[f(\textbf{x}_{1:t})]
\end{align*}\end{scriptsize}


Since $\sum_{t=1}^T\mathbb{E}_{\textbf{x}_{1:T}}[f(\textbf{x}_{1:t})]=\mathbb{E}_{\textbf{x}_{1:T}}[\sum_{t=1}^T f(\textbf{x}_{1:t})]=L_{score}$, we prove the corollary. \hfill $\square$

\subsection{Datasets and Baselines}

We use 5 datasets from various fields as follows. We summarize their data dimensions, the number of training samples, and their time-series lengths (window sizes) in Table~\ref{table1}.

\begin{itemize}
    \item \textit{Stock}~\cite{yoon2019timegan}: The Google stock dataset was collected irregularly from 2004 to 2019. Each observation has (volume, high, low, opening, closing, adjusted closing prices), and these features are correlated.
   \item \textit{Energy}~\cite{candanedo2017energy}:
   This dataset is from the UCI machine learning repository for predicting the energy use of appliances from highly correlated variables such as house temperature and humidity conditions.
    \item \textit{Air}~\cite{devito2008air}: The UCI Air Quality dataset were collected from 2004 to 2005. Hourly averaged air quality records are gathered using gas sensor devices in an Italian city.
    \item \textit{AI4I}~\cite{matzka2020}: AI4I means the UCI AI4I 2020 Predictive Maintenance dataset. This data reflects the industrial predictive maintenance scenario with correlated features including several physical quantities.
    \item \textit{Occupancy}~\cite{candanedo2016occupancy}: The UCI Room Occupancy Estimation dataset contains features such as temperature, light, and CO2 collected from multiple environmental sensors.
\end{itemize}

We use several types of generative methods for time-series as baselines. At first, we consider autoregressive generative methods: T-Forcing (teacher forcing)~\cite{graves2013tforcing,ilya2011tforcing} and P-Forcing (professor forcing)~\cite{lamb2016pforcing}. 
Next, we use GAN-based methods: TimeGAN~\cite{yoon2019timegan}, RCGAN~\cite{esteban2017rcgan}, C-RNN-GAN~\cite{mogren2016crnngan}, COT-GAN~\cite{xu2020cotgan}, WaveGAN~\cite{donahue2019wavegan}. Finally, we also test VAE-based methods into our baselines: TimeVAE~\cite{desai2021timevae}.

\begin{table}[h]
\centering
\begin{tabular}{c|c|c|c}
    \hline
    Dataset & Dimension & \#Samples & Length \\
    \hline
    Stocks & 6 & 3685 & \multirow{5}{1em}{24}\\
    Energy & 28 & 19735 &\\
    Air & 13 & 9357 &\\
    AI4I & 5 & 10000 &\\
    Occupancy & 13 & 10129 &\\
    \hline
\end{tabular}
\caption{Characteristics of the datasets we use for our experiments}
\label{table1}
\end{table}

\subsection{Search Space for Hyperparameters}

We give our search space for hyperparameters. $iter_{pre}$, is in \{50000,100000\}. The dimension of hidden features, $d_{hidden}$, ranges from 2 times to 5 times the dimension of input features. For other settings, we follow the default values in TimeGAN~\cite{yoon2019timegan} and VPSDE~\cite{song2021SDE}.

\subsection{Miscellaneous Experimental Environments}

We give detailed experimental environments. The following software and hardware environments were used for all experiments: \textsc{Ubuntu} 18.04 LTS, \textsc{Python} 3.9.12, \textsc{CUDA} 9.1, and \textsc{NVIDIA} Driver 470.141, and i9 CPU, and \textsc{GeForce RTX 2080 Ti}.

On the experiments, We report only the VP and subVP-based TSGM and exclude the VE-based one for its lower performance. For baselines, we reuse their released source codes in their official repositories and rely on their designed training and model selection procedures. For our method, we select the best model every 5000 iteration. For this, we synthesize samples, and calculate the mean and standard deviation scores of the discriminative and predictive scores.

\subsection{Empirical Space and Time Complexity Analyses}

We report the memory usage during training in Table~\ref{table7} and the wall-clock time for generating 1,000 time-series samples in Table~\ref{table8}. We compare only with TimeGAN~\cite{yoon2019timegan} since is the state-of-the-art method. Our method is relatively slower than TimeGAN, which is a fundamental drawback of all SGMs. For example, the original score-based model~\cite{song2021SDE} requires 3,214 seconds for sampling 1,000 CIFAR-10 images while StyleGAN~\cite{DBLP:journals/corr/abs-1912-04958} needs 0.4 seconds. However, we also emphasize that this problem can be relieved by using the techniques suggested in~\cite{xiao2022tackling,DBLP:journals/corr/abs-2105-14080} as we mentioned in the conclusion section.


\begin{table}[h]
\centering
\begin{tabular}{c|c|c}
    \hline
    Model & Stock & Energy \\
    \hline
    TimeGAN & 1.1(GB) & 1.6(GB)\\
    TSGM & 3.8(GB) & 3.9(GB)\\
    \hline
\end{tabular}
\caption{The memory usage for training}
\label{table7}
\end{table}

\begin{table}[h]
\centering
\begin{tabular}{c|c|c}
    \hline
    Model & Stocks & Energy \\
    \hline
    TSGM & 3318.99(s) & 1620.84(s)\\
    TimeGAN & 0.43(s) & 0.47(s)\\
    \hline
\end{tabular}
\caption{The sampling time of TSGM and TimeGAN for generating 1,000 samples on each dataset. The original score-based model~\cite{song2021SDE} requires 3,214 seconds for sampling 1000 CIFAR-10 images while StyleGAN~\cite{DBLP:journals/corr/abs-1912-04958} needs 0.4 seconds, which is a similar ratio to our results.}
\label{table8}
\end{table}

\end{document}